\documentclass[letterpaper, 10 pt, conference]{ieeeconf}

\IEEEoverridecommandlockouts
\usepackage{makecell}
\usepackage{algorithm}
\usepackage[noend]{algpseudocode}
\usepackage{array}
\usepackage[caption=false,font=normalsize,labelfont=sf,textfont=sf]{subfig}
\usepackage{textcomp}
\usepackage{stfloats}
\usepackage{url}
\usepackage{xcolor}
\usepackage{verbatim}
\usepackage{graphicx}

\usepackage{amsthm}
\usepackage{amssymb,mathtools}
\usepackage{cite}
\usepackage{optidef}
\usepackage{multirow}
\usepackage{booktabs}
\usepackage{siunitx}
\usepackage{threeparttable}
\usepackage{xurl}
\usepackage[hidelinks]{hyperref}

\title{\LARGE \bf
SHIELD: Scalable Optimal Control with Certification using Duality and Convexity}

\author{Hansung Kim$^1$, Siddharth H. Nair$^2$, and Francesco Borrelli$^1$
\thanks{$^1$HK, FB are with the Model Predictive Control Laboratory, UC Berkeley. E-mails: 
\{hansung, fborrelli\}@berkeley.edu}
\thanks{$^2$SHN is with Nextracker Inc.}
}

\makeatletter
\algrenewcommand\alglinenumber[1]{\footnotesize #1}
\makeatother

\theoremstyle{plain}
\newtheorem{assumption}{Assumption}

\newtheorem{theorem}{Theorem}
\newtheorem{Definition}{Definition}
\newtheorem{proposition}{Proposition}
\newtheorem{corollary}{Corollary}

\begin{document}

\maketitle
\thispagestyle{empty}
\pagestyle{empty}

\begin{abstract}

We present SHIELD, a hierarchical algorithm that reduces both the decision-variable dimension and the constraint set in $\ell_1$-regularized convex programs. From strong convexity and Lagrangian duality, we derive certificates that \emph{safely} discard constraints and decision variables while guaranteeing that all removed constraints remain satisfied and all removed variables are null. To further accelerate the proposed algorithm, we propose a transformer-based deep neural network to guide the dual certificate inference. We validate SHIELD on stochastic model predictive control (SMPC) in complex, multi-modal traffic scenarios, comparing against a full-dimensional SMPC policy. Numerical simulations demonstrate order-of-magnitude computational speedups while preserving feasibility and closed-loop safety, highlighting the practicality of certifiably safe, lightweight MPC in complex driving scenes. GitHub: \url{https://github.com/MPC-Berkeley/SHIELD%5Fmpc}.

\end{abstract}
\section{Introduction}
Deep learning and other data-driven techniques have grown rapidly in modern control, with broad adoption in robotics, power systems, chemical processes, and manufacturing. Learning is commonly used for model identification, uncertainty quantification, controller performance improvement, and accelerating online control computation. This last direction has become particularly active in optimization-based control, especially Model Predictive Control (MPC). At each control step, MPC solves a constrained optimization problem to compute the next input while accounting for dynamics, physical limits, and safety margins. However, this repeated optimization must run at real-time rates—often milliseconds—which becomes challenging as problem complexity grows, e.g., longer horizons, richer models, and hundreds of inequality constraints. In dynamic, uncertain environments, such computational latency can directly degrade performance or compromise safety.
\begin{figure} \label{fig:shield}
    \centering
    \includegraphics[width=0.85\linewidth]{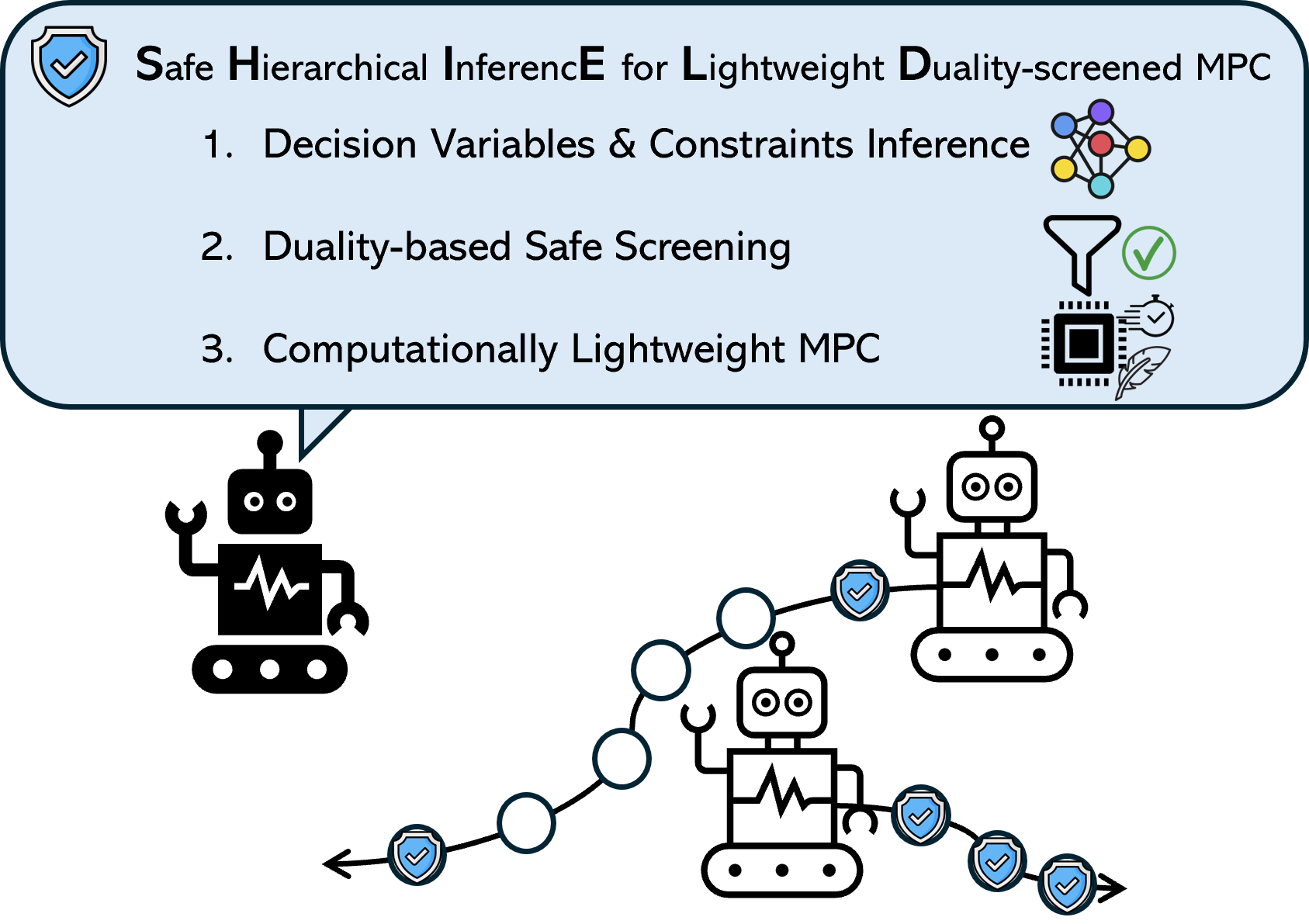}
    \caption{An illustration of the Safe Hierarchical Inference for Lightweight Duality-Screened MPC (SHIELD). The controlled robot (black) applies the SHIELD algorithm to infer about which constraints and decision variables are safe to ignore for computationally lightweight model predictive control.}
    \label{fig:placeholder}

\end{figure}

Recent supervised-learning frameworks address this computational bottleneck via the “learning to optimize” paradigm~\cite{Sacks_2022,arango}, embedding neural priors into solver pipelines to achieve runtime gains, particularly in large-scale or embedded control systems. These methods exploit structure in MPC solutions and have been studied for convex and mixed-integer convex problems. Broadly, they fall into two approaches trading off speed and optimality. The first learns warm-starts, predicting initial guesses to reduce solver iterations~\cite{Stellato_2017,quirynen,Marcucci}. The second predicts complete or partial solutions, learning mappings from problem parameters to optimizers~\cite{primal-dual,lampos,coco,cauligi2022learningmixedintegerconvexoptimization}. Warm-starting is typically more robust to mispredictions and preserves certifiable optimality by solving the full problem; however, in time-varying settings, it is difficult to predict warm-starts that consistently reduce iterations. Solution-prediction methods can be fast (via direct prediction or reduced online solves after eliminating variables/constraints), and recent work quantifies suboptimality~\cite{primal-dual,lampos}; nevertheless, mispredictions can cause suboptimality or infeasibility. We therefore target a robust approach for accelerating real-time convex MPC with time-varying constraints, variables, and cost, as in multi-robot coordination, energy-aware fleet control, and autonomous driving.

Our work is related to the safe screening literature in static optimization~\cite{elghaoui2010,bonnefoy2015,fercoq2015}, which exploits feasible dual points—typically from a regularization path over decreasing $\lambda$—to certify inactive features or constraints before solving. These methods assume a fixed problem instance with a cheaply available high-quality dual. In contrast, MPC requires re-screening at every control step as constraints, cost, and parameters change, with no regularization path to exploit. Our approach differs in that (i) $\zeta$-tightening certifies feasibility of the \emph{original} untightened constraints, not merely optimality-based inactivity; (ii) a learned classifier predicts the dual support structure so that only a reduced-dimension dual is solved, avoiding the cost of a full dual solve; and (iii) the duality-gap bound provides valid certificates from any feasible dual, including approximate ones, enabling robust screening under distribution shift inherent in time-varying control.

Building on previous work~\cite{kim_iv}, we propose a hierarchical learning-based approach for scalable constrained convex optimal control in dynamic scenarios.
Our approach has two components: (i) an attention-based deep neural network (DNN), trained on time-varying data, that predicts \emph{marginally important} \textit{variables and constraints} for optimal control at run-time; and (ii) a Lagrange duality and strong convexity-based safety screen that certifies whether inferred variables and constraints can be safely eliminated without violating original constraints. By marginally important, we mean variables and constraints whose presence or absence does not affect the optimal planned trajectory.

The result is a lightweight convex optimal control problem with fewer decision variables and constraints, enabling fast run-time solves. 
Rather than relying on \emph{a priori} statistical guarantees for the neural network, we use the duality-based certification algorithm online to maintain robustness under train-to-test distribution shift in time-varying scenarios. The resulting algorithm applies to any convex program with Lipschitz constraints and satisfying Slater conditions, independent of the control-specific structure. 

Sec.~\ref{sec:prob} formulates a generic time-varying convex optimal control problem. 
Sec.~\ref{sec:method} presents our \textbf{S}afe \textbf{H}ierarchical \textbf{I}nferenc\textbf{E} for \textbf{L}ightweight, \textbf{D}uality-screened (SHIELD) MPC algorithm for safe computational acceleration, and Sec.~\ref{sec:example} demonstrates its effectiveness in multi-modal traffic MPC.

\section{Problem Statement} \label{sec:prob}
We consider a convex, finite-horizon optimal control problem that must be solved at a high rate to enable real-time robot control. In a standard MPC setup, at each time step $t$, the robot solves the following optimization problem:
{\small{
\begin{subequations}\label{opt:convex_opt}
\begin{align}
 \min_{\mathbf{x}_t,\mathbf{u}_t}&\quad f^0_t(\mathbf{x}_t,\mathbf{u}_t)
 \label{eq:cost}\\
 \text{s.t. }&\quad f^i_t(\mathbf{x}_t,\mathbf{u}_t) \leq 0, \quad \forall i\in\mathbb{I}_1^{m} \label{eq:fi}\\
 &\quad h^i_t(\mathbf{x}_t,\mathbf{u}_t) = 0, \quad \forall i\in\mathbb{I}_1^{p}
 \label{eq:equality}
\end{align}
\end{subequations}}}to compute feedback policy $\pi_t(x_t)=u^\star_{t|t}$, where $x_{k|t}\in\mathbb{R}^{n_x}$ and $u_{k|t}\in\mathbb{R}^{n_u}$ are state and input at prediction step $k$ from time $t$.
Let $\mathbf{x}_t$ and $\mathbf{u}t$ stack states and inputs over horizon $N$, and denote ${k_1,\dots,k_2}$ by $\mathbb{I}_{k_1}^{k_2}$.
We assume $f_t^0$ is strongly convex in $(\mathbf{x},\mathbf{u})$.
Functions $f_t^i$ in~\eqref{eq:fi} are convex inequality constraints, and $h_t^i$ in~\eqref{eq:equality} encode linear dynamics and initial conditions.
\eqref{opt:convex_opt} has $m$ inequality and $p$ equality constraints.

Computing $\pi_t(\cdot)$ by solving \eqref{opt:convex_opt} at high rate becomes challenging at large problem scales, even on modern CPUs.
In heterogeneous multi-agent, multi-modal settings, constraints can grow combinatorially with the number of modes and surrounding agents, leading to prohibitive solve times~\cite{kim_iv,nair_mm_mpc}. 
Likewise, centralized (and some distributed) optimization-based multi-robot controllers scale poorly with the number of robots~\cite{drobustmpc,conflict_mpc,centralized_multi}. 
When constraints and decision variables grow superlinearly with horizon, agents, or scenario complexity, even fast solvers may not deliver low-latency control.

Our objective is to accelerate computation of $\pi_t(\cdot)$ in real-time while certifying safety, i.e., satisfaction of the original constraints in~\eqref{opt:convex_opt}.
Using certifiable screening, we solve a reduced convex program that provides a feasible solution to the original problem by construction. Certifiability is defined as preserving feasibility of the original convex program \eqref{opt:convex_opt}.
\section{Safe Hierarchical Inference for Duality-Screened Lightweight Optimal Control} \label{sec:method}
\subsection{$\ell_1$ Regularized Optimization Problem} \label{sec:reg_prob}
Without loss of generality, we adopt an affine control policy,
\(
u_{k|t}=\pi_{\boldsymbol{\theta}_t}\!\left(x_{k|t},\, z_t\right),
\)
where $u_{k|t}$ is the control at prediction step $k$ given information at time $t$,
$x_{k|t}$ is the predicted state, and $z_t$ collects environment parameters (e.g., obstacle states).
We optimize over the policy parameters $\boldsymbol{\theta}_t$.
Such parameterizations are standard in MPC; in uncertain multi-modal settings they combine state feedback
with affine disturbance feedback (ADF), improving feasibility and reducing conservatism~\cite{nair_mm_mpc}.
For open-loop input-sequence optimization, the feedback parameters are zero and $\boldsymbol{\theta}_t$ is the open-loop sequence.

Let the decision variable $\boldsymbol{\theta}_t$ collect the control policy parameters across the prediction horizon. We add $N$ additional equality constraints for $u_{k|t}=\pi_{\boldsymbol{\theta}_t}\!\left(x_{k|t},\, z_t\right)$ to $h^i_t$ represented by \eqref{eq:l1_equality} and add $l_1$ regularization to the cost function:{\small{
\begin{subequations}\label{opt:l1_general}
\begin{align}
 \min_{\mathbf{x}^i_t,\mathbf{u}^i_t,\boldsymbol{\theta}_t}&\quad f^0_t(\mathbf{x}_t,\mathbf{u}_t) + \lambda \lVert S\boldsymbol{\theta}_t\rVert_1
 \label{eq:l1_cost}\\
 \text{s.t. }&\quad \Bar{f}^i_t(\mathbf{x}_t,\mathbf{u}_t) \leq 0, \quad \forall i\in\mathbb{I}_1^{m-c} \label{eq:l1_fi}\\
 &\quad \Tilde{f}^i_t(\mathbf{x}_t,\mathbf{u}_t) \leq 0, \quad  \forall i\in\mathbb{I}_1^{c} \label{eq:fi_constr}\\
 &\quad h^i_t(\mathbf{x}_t,\mathbf{u}_t) = 0, \quad \forall i\in\mathbb{I}_1^{p+N} 
 \label{eq:l1_equality}
\end{align}
\end{subequations}}}where $S$ is a user-defined selection matrix to choose which control policy parameters to sparsify. The $l_1$ penalty promotes sparsity in the decision variables and feature selection, which enables optimization over decision variables in control-task importance \cite{candes_enhancing_2008}.
We partition \eqref{eq:fi} into (i) $\Bar{f}^i_t$, \emph{immutable}—always enforced (physical limits, non-negotiable safety: actuator/state limits; friction/flight envelopes; traffic rules/road bounds) and (ii) $\Tilde{f}^i_t$, \emph{screenable}—relevant only when near-active (collision avoidance with distant agents/obstacles; keep-out/no-fly regions irrelevant to current path). For instance, constraint $\Tilde{f}^i_t\leq 0(\mathbf{x}_t,\mathbf{u}_t)$ can be screened out (i.e. removed from optimization problem solved at time $t$) if a strict margin $\Tilde{f}^i_t(\mathbf{x}_t,\mathbf{u}_t) < 0$ for $i$ at the optimal solution is certified.

This provides an opportunity to reduce the $c$ constraints where the number of inequality constraints $m\geq c\gg1$ limits the fast computation of $\eqref{opt:l1_general}$.

For a linear system, given the initial condition on $x_{t|t}$, and the affine control policy parameterization $\boldsymbol{\pi_{\theta}}$, and composing the functions with the affine map $\boldsymbol{\theta}_t\mapsto(\mathbf{x}_t(\boldsymbol{\theta}_t),\mathbf{u}_t(\boldsymbol{\theta}_t))$, we can rewrite \eqref{opt:l1_general} as follows via recursive substitution: {\small{ 
\begin{subequations}\label{opt:l1_theta}
\begin{align}
 \min_{\boldsymbol{\theta}_t}&\quad f^0_t(\boldsymbol{\theta}_t) + \lambda \lVert S\boldsymbol{\theta}_t\rVert_1
 \label{eq:l1_theta_cost}\\
 \text{s.t. }&\quad \Bar{f}_i(\boldsymbol{\theta}_t) \leq 0, \quad \forall i\in\mathbb{I}_1^{m-c} \label{eq:l1_theta_fi}\\
 &\quad \Tilde{f}^i_t(\boldsymbol{\theta}_t) \leq -\zeta, \quad  \forall i\in\mathbb{I}_1^{c} \label{eq:l1_theta_fi_constr}\\
 &\quad h^i_t(\boldsymbol{\theta}_t) = 0, \quad \forall i\in\mathbb{I}_1^{p},
 \label{eq:l1_theta_equality}
\end{align}
\end{subequations}}}
where $\zeta>0$ is a \textit{safety-margin} that we desire on constraint satisfaction. We tighten the \emph{screenable} constraints to enable a design of screening condition that guarantees at most $\zeta$ violation of the tightened constraint, ensuring $\tilde f_t^{i}(\boldsymbol{\theta}_t)\le 0$.

\begin{assumption} \label{assumption:strong}
    $f^0_t(\boldsymbol{\theta}_t)\in C^2(\mathbb{R}^n)$ is $\underline{\sigma}$-strongly convex w.r.t. $\boldsymbol{\theta}_t$ and has $\Bar{\sigma}$-Lipschitz continuous gradient. 
\end{assumption}

Problem \eqref{opt:l1_theta} after the epigraph reformulation is
{\small{
\begin{subequations}\label{opt:l1_epigraph}
\begin{align}
 \min_{\boldsymbol{\theta}_t,s}&\quad f^0_t(\boldsymbol{\theta}_t) + \lambda \mathbf{1}^\top s
 \label{eq:l1_epigraph_cost}\\
 \text{s.t. }&\quad \Bar{f}^i_t(\boldsymbol{\theta}_t) \leq 0, \quad \forall i\in\mathbb{I}_1^{m-c} \label{eq:l1_epigraph_fi}\\
 &\quad \Tilde{f}^i_t(\boldsymbol{\theta}_t) \leq -\zeta, \quad \forall i\in\mathbb{I}_1^{c} \label{eq:l1_epigraph_fi_constr}\\
 &\quad h^i_t(\boldsymbol{\theta}_t) = 0, \quad \forall i\in\mathbb{I}_1^{p}  \label{eq:l1_epigraph_equality}\\
  &\quad S\boldsymbol{\theta}_t-s\leq 0, -S\boldsymbol{\theta}_t -s \leq 0, s\geq 0. \label{eq:epi}
\end{align}
\end{subequations}}}

Let the dual variables $\mu_i,\eta_i,\nu_i, g_1, g_2, \gamma$ correspond to $\Tilde{f}^i_t, \Bar{f}^i_t, h^i_t$, and \eqref{eq:epi}, respectively.
Let $\Phi\coloneqq f^0_t + \sum \mu_i (\Tilde{f}^i_t+\zeta) + \sum \eta_i \Bar{f}^i_t + \sum \nu_i h^i_t$. 
The dual problem is given by
{\small{
\begin{equation}\label{opt:dual_problem}
 \min_{\mu\geq0,\eta\geq0,\nu,\lVert g\rVert_\infty\leq \lambda} \quad -\inf_{\boldsymbol{\theta}_t}\{\Phi (\boldsymbol{\theta}_t; \mu,\eta,\nu) + (S^\top g)^\top \boldsymbol{\theta}_t \} 
\end{equation}}}

For well-defined convex optimization problems, $\Phi$ can be computed explicitly. For example, if $f^0_t$ is quadratic and $\Tilde{f}^i_t,\Bar{f}^i_t,h^i_t$ are affine in $\boldsymbol{\theta}_t$ and \eqref{opt:dual_problem} becomes a quadratic program.

The optimal dual solutions $(\boldsymbol{\theta}_t^\star,\mu^\star,\eta^\star,\nu^\star,g^\star)$ are given by the KKT conditions:
{\small{
\begin{align}  \label{eq:kkt}
& \mathbf{0}\in \partial f^0_t(\boldsymbol{\theta}_t^\star) + \sum_i \mu_i^\star \partial \Tilde{f}^i_t(\boldsymbol{\theta}_t^\star) + \sum_i \eta_i^\star \partial \Bar{f}^i_t(\boldsymbol{\theta}_t^\star) \\ 
&+\sum_i \nu_i^\star \partial h^i_t(\boldsymbol{\theta}_t^\star) + S^\top g^\star, g^\star \in \lambda \partial \lVert S \boldsymbol{\theta}_t^\star \rVert_1 \nonumber\\
&\Tilde{f}^i_t(\boldsymbol{\theta}_t^\star) \leq -\zeta, \; \mu_i^\star \geq 0, \; \forall i \in \mathbb{I}_1^c, h^i_t(\boldsymbol{\theta}_t^\star) = 0, \; \forall i \in \mathbb{I}_1^p\nonumber\\ &\Bar{f}^i_t(\boldsymbol{\theta}_t^\star) \leq 0, \; \eta_i^\star \geq 0, \; \forall i \in \mathbb{I}_1^{m-c}, \; ||g^\star||_\infty\leq \lambda \nonumber\\
& \mu_i^\star (\Tilde{f}^i_t(\boldsymbol{\theta}_t^\star)+\zeta) = 0,\; \forall i \in \mathbb{I}_1^c,\; \eta^\star_i \Bar{f}^i_t(\boldsymbol{\theta}_t^\star) = 0,\; \forall i \in \mathbb{I}_1^{m-c} \nonumber \\
& {g^\star}^\top S \boldsymbol{\theta}_t^\star = \lambda \lVert S\boldsymbol{\theta}_t^\star\rVert_1 \nonumber
\end{align}}}
\begin{proof}
    See Appendix \ref{appendix:kkt}
\end{proof}
Observe from the last two equations that \begin{enumerate}
    \item If $\mu_i^\star=0$, then $\tilde{f}^i_t(\boldsymbol{\theta}_t^\star)<-\zeta$ and this constraint can be removed from \eqref{opt:l1_theta}, 
    \item If $\lVert g_i^{\star}\rVert_\infty < \lambda$, then $[\mathcal{S}\boldsymbol{\theta}_t^\star]_i=0$ and the variables corresponding to it can be removed from \eqref{opt:l1_epigraph}.
\end{enumerate}
where $g_i$ is the dual variable corresponding to the $i$-th components of the $\ell_1$ regularization term \eqref{eq:l1_epigraph_cost} and $[S\boldsymbol{\theta}_t^\star]_i$ denotes the $i$-th component of $S\boldsymbol{\theta}_t^\star$. As small nonzeros from tolerances/degeneracy make a hard zero-test brittle, we seek for more robust and useful constraint screening condition.

\subsection{Constraint Screening Condition}
We utilize perturbation/sensitivity analysis of convex programs to derive constraint-removal conditions.
Suppose we obtain solution $\bar{\boldsymbol{\theta}}_t$ after removing constraint $\Tilde{f}^i_t$.
The violation of constraint $\Tilde{f}^i_t$, denoted $\delta_i$, is analogous to constraint perturbations in \cite{boyd2004convex}.
Let $p_t(\delta_i)$ be the objective value of \eqref{opt:l1_reduced} at $\bar{\boldsymbol{\theta}}_t$, and $p_t(0)$ the optimal value of original problem~\eqref{opt:l1_epigraph}.

\begin{assumption}
Strong duality holds for~\eqref{opt:l1_epigraph} under Slater's condition~\cite{boyd2004convex}.
\end{assumption}
By the global sensitivity result in~\cite[Ch.~5]{boyd2004convex},
\begin{equation}\label{eq:sensitivity}
p_t(0)-p_t(\delta_i)\le \delta_i \mu_i^\star,
\end{equation}
which quantifies how the optimal value changes when relaxing/removing constraint $i$. 
\begin{assumption}\label{assump:lipschitz}
Each constraint $\Tilde f^i_t(\cdot)$ is $L_i$-Lipschitz.
\end{assumption}

Even if the reduced convex program \eqref{opt:l1_reduced} obtained by removing $\Tilde f^i_t(\cdot)$ yields a solution that is infeasible for the tightened problem~\eqref{opt:l1_epigraph}, it remains within the safety margin $\zeta$ and satisfies the original constraints \eqref{eq:fi_constr}, as formalized below.

\begin{proposition} \label{prop1}
Let Assumptions \ref{assumption:strong}--\ref{assump:lipschitz} hold and
the solution to \eqref{opt:l1_general} $\boldsymbol{\theta}_t^\star$ be strictly feasible with margin
$\zeta \geq \min_i \{-\Tilde{f}^i_t(\boldsymbol{\theta}_t^\star)\} > 0$ for all $i$. Denote $f^0_t=\eqref{eq:l1_epigraph_cost}$
and define
{\small{\[
\varepsilon_{\mathrm{crit}}
\;:=\;
\frac{\underline{\sigma}}{2}\,\min_i \Big(\frac{\zeta}{L_i}\Big)^2.
\]}}
Then, any $\boldsymbol{\theta}_t$ with $|f^0_t(\boldsymbol{\theta}_t)-f^0_t(\boldsymbol{\theta}^\star_t)|\leq \varepsilon_{\mathrm{crit}}$ is feasible for the
original constraints \eqref{eq:fi_constr}, i.e., $\Tilde{f}^i_t(\boldsymbol{\theta}_t)\leq 0$ for all $i$.
\end{proposition}
\begin{proof}
Let $|f^0_t(\boldsymbol{\theta}_t) - f^0_t(\boldsymbol{\theta}^\star_t)| \leq \epsilon$ and by $\underline{\sigma}$-strong convexity,
{\small{
\[
f^0_t(\boldsymbol{\theta}_t)-f^0_t(\boldsymbol{\theta}^\star_t)\;\ge\;\frac{\underline{\sigma}}{2}\,\|\boldsymbol{\theta}_t-\boldsymbol{\theta}^\star_t\|^2
\quad\Rightarrow\quad
\|\boldsymbol{\theta}_t-\boldsymbol{\theta}^\star_t\|\;\le\;\sqrt{\tfrac{2\,\epsilon}{\underline{\sigma}}}.
\]}}
Since $\Tilde{f}^i_t$ is $L_i$-Lipschitz,
{\small{\[
\Tilde{f}^i_t(\boldsymbol{\theta}_t)\;\le\; \Tilde{f}^i_t(\boldsymbol{\theta}^\star_t)+L_i\|\boldsymbol{\theta}_t-\boldsymbol{\theta}^\star_t\|
\;\le\; -\zeta + L_i\sqrt{\tfrac{2\,\epsilon}{\underline{\sigma}}}\leq 0.
\]}}
Thus, $\Tilde{f}^i_t(\boldsymbol{\theta}_t)\le 0$ holds whenever
$\epsilon \leq \tfrac{\underline{\sigma}\,\zeta^2}{2L_i^2}$ and enforcing this for all $i$
gives $\epsilon \leq \epsilon_{\mathrm{crit}}$, which proves feasibility.
\end{proof}


\begin{corollary}\label{col:constr}
From Proposition~\ref{prop1} and~\eqref{eq:sensitivity}, we set $\boldsymbol{\delta}_i$ to $\zeta$ and get
\(|p_{t}(0)-p_t(\zeta)|\le \zeta\|\mu_i^\star\|_2\) where the solution to the $\zeta$-perturbed primal is $\Bar{\boldsymbol{\theta}}_t$.
Thus, if $\|\mu_i^\star\|_2 \le \epsilon/\zeta$ (i.e., perturbation is $\zeta$ and the change in the optimal cost is bounded by $\epsilon$) and $\epsilon\leq\frac{\underline{\sigma}\zeta^2}{2}\min_i \frac{1}{L_i^2}$, then the $i$-th constraint can be safely removed and the original constraint $\Tilde{f}^i_t(\Bar{\boldsymbol{\theta}}_t)\leq 0$ is satisfied.    
\end{corollary}
Because the derived screening conditions in Sec.~\ref{sec:reg_prob} apply to optimal duals, we generalize the inequalities and Cor.~\ref{col:constr} to any feasible dual via a dual optimality-gap bound, enabling certifiable safe screening with an approximate dual computed from DNN class predictions and a linear equation solve.

\subsection{Duality-based Safe Screening Certification}
We derive certifiably \emph{safe} screening conditions for constraints and decision variables in~\eqref{opt:l1_epigraph} by leveraging strong convexity of the primal and strong duality.
\begin{proposition}\label{prop:dual_smooth_strong}
Under Assumption~\ref{assumption:strong}, the dual objective
$d(\mu,\eta,\nu,g)$ in~\eqref{opt:dual_problem} is $\bar{\rho}$-smooth and $\underline{\rho}$-strongly convex over
the feasible set $\mathcal{Y}$ \footnote{We use the minimization form of the dual in~\eqref{opt:dual_problem} (i.e., the negative of the standard concave dual), hence convexity.}.

In the quadratic-program case where the primal cost
$f^0_t$ has Hessian $Q_t \succ 0$, the dual Hessian is $Q_t^{-1}$, and
\[
\bar{\rho} = \lambda_{\max}(Q_t^{-1}), \qquad
\underline{\rho} = \lambda_{\min}(Q_t^{-1}).
\]
\end{proposition}
For brevity, we stack the dual variables as $y \coloneqq [\mu,\eta,\nu,g]$.

\begin{Definition}
    The projected gradient for the dual is 
    {\small{\begin{equation}
    \nabla^\dagger d(y)= y-[y-\nabla d(y)]_{\mathcal{Y}}, \;\forall{y\in\mathcal{Y}} \nonumber
\end{equation}}}
 where $\mathcal{Y}$ is the dual feasible set, which is nonempty, closed, and convex from \eqref{opt:dual_problem}. The $[x]_{\mathcal{Y}}$ is a convex projection operator. 
\end{Definition}
The gradient of the dual cost can be computed explicitly offline and may be a parametric function, which can be quickly evaluated online. Note, at the optimal dual solution $y^\star$, $\nabla^\dagger d(y^\star)=0$ \cite{wang2014iteration}. Then, the following inequality estimates the distance of a feasible dual $y$ from optimality \cite{wang2014iteration}:
{\small{
\begin{align}\label{eq:pred_err_bound}
\lVert y-y^\star \rVert_2 \leq \frac{1+\bar{\rho}}{\underline\rho}||\nabla^\dagger d(y)||_2.
\end{align}}}
Since $\nabla^\dagger d(y^\star)=0$, the bound \eqref{eq:pred_err_bound} is tight. Define 
{\small{
\begin{align}\label{eq:gap}
\textit{gap}(\mu,\eta,\nu,g)\coloneqq\frac{1+\bar{\rho}}{\underline\rho}||\nabla^\dagger d(\mu,\eta,\nu,g)||_2.
\end{align}}}

\begin{theorem} \label{th:screen}
Given a feasible dual solution $(\mu,\eta,\nu,g)$ to \eqref{opt:dual_problem},
\begin{enumerate}
\item  If $\lVert \mu_i\rVert_2+\text{gap}(\mu,\eta,\nu,g) \leq \epsilon/\zeta$ and $\epsilon\leq\frac{\underline{\sigma}\zeta^2}{2}\min_i \frac{1}{L_i^2}$, then $\Tilde{f}^i_t$ can be safely removed from \eqref{opt:l1_epigraph}.

\item If $\lVert g_i\rVert_\infty +\text{gap}(\mu,\eta,\nu,g)< \lambda$, then the variables $[S\boldsymbol{\theta}_t]_i$ comprising the policy parameters are zero and thus can be safely removed from \eqref{opt:l1_epigraph}.
\end{enumerate}
\end{theorem}

\begin{proof}
Statement (1) follows from Corollary \ref{col:constr}:
{\small{
\begin{align} 
    \lVert \mu^{\star}_i\rVert_2 \leq &~\lVert \mu_{i}\rVert_2 + \lVert \mu_i^{\star}- \mu_i\rVert_2\nonumber\\
    \leq&~\lVert \mu_i\rVert_2 +\textit{gap}(\mu,\eta,\nu,g)<
    \epsilon / \zeta \nonumber
\end{align}}}
and if $\frac{\underline{\sigma}\zeta^2}{2}\min_i \frac{1}{L_i^2}$, then $\tilde{f}^i_t$ can be \emph{safely} removed.
If $\epsilon \leq \lVert g_i\rVert_\infty+\text{gap}(\mu,\eta,\nu,g) < \lambda$ some $i\in\mathbb{I}_1^{|S\boldsymbol{\theta}_t|}$, then we have
{\small{
\begin{align} 
    \lVert g^{\star}_i\rVert_\infty \leq &~\lVert g_{i}\rVert_\infty + \lVert g_i^{\star}- g_i\rVert_2\nonumber\\
    \leq&~\lVert g_i\rVert_\infty +\textit{gap}(\mu,\eta,\nu,g)<
    \lambda\nonumber
\end{align}}}
and consequently, the variables $[S\boldsymbol{\theta}_t]_i=0$ from \eqref{eq:kkt} and can be \textit{safely} removed from \eqref{opt:l1_epigraph}.
\end{proof}
\subsection{SHIELD Algorithm}
Algorithm~\ref{alg:hmpc} outlines the procedure to obtain the reduced problem from \eqref{opt:l1_epigraph}.
\subsubsection{Feasible Dual Solution Estimation}
Given any feasible dual solution $(\hat{\mu},\hat{\eta},\hat{\nu},\hat{g})$, we apply Theorem~\ref{th:screen} to \emph{safely} eliminate variables and constraints using the dual optimality-gap upper bound in~\eqref{eq:gap}. 
When \eqref{opt:l1_epigraph} is large, directly solving the dual \eqref{opt:dual_problem} can be expensive; instead, we minimize the \emph{unconstrained} dual objective and project the minimizer onto the convex dual-feasible set of~\eqref{opt:dual_problem} to obtain a feasible dual estimate. 
We then use~\eqref{eq:gap} together with Theorem~\ref{th:screen} to certify when this estimate is reliable enough to prune variables and constraints.
Even this unconstrained solve can be costly at scale (Sec.~\ref{sec:results}). 
To further accelerate, we use a neural-network classifier to preselect a subset of dual variables, yielding a reduced-dimensional dual problem.

The dual classifier $\pi^{\text{class}}_{\zeta}(\mathbf{z}_t)$, specific to a choice of $\zeta$, predicts the dual class
$(\tilde{\mu},\tilde{g})$ from features $\mathbf{z}_t$ encoding the environment and
parameters of~\eqref{opt:l1_epigraph}. The feature design, architecture, and training
procedure are problem-dependent.
The resulting class vector categorizes the $i$-th dual as follows:
{\small{
\begin{equation} \label{eq:classification}
\mu_i = \begin{cases}
>0, & \Tilde{\mu}_i=1,\\
0, & \Tilde{\mu}_i=0,
\end{cases}
\;
\lVert g_i \rVert_\infty = \begin{cases}
\lambda, & \Tilde{g}_i=1,\\
<\lambda, & \Tilde{g}_i=0.
\end{cases} 
\end{equation}}}
Using the classifier outputs, we reduce the dual by selecting only
\(\tilde{\mu}_i=1\) and \(\tilde{g}_i=0\) entries. Let \(S^r_\mu\) and \(S^r_g\) be the
corresponding selection matrices; the reduced variables are
{\small{
\begin{subequations}\label{eq:reduced-duals}
\begin{align}
\mu^r &= S^r_\mu\,\mu, \\
g^r   &= S^r_g\,g.
\end{align}
\end{subequations}}}
After solving the reduced unconstrained dual for \((\mu^r,\eta,\nu,g^r)\),
we lift \((\mu^r,g^r)\) to full dimension by inserting zeros at indices $i$ with
\(\tilde{\mu}_i=0\) and \(\lambda\) at indices with \(\tilde{g}_i=1\), consistent with
\eqref{eq:classification}. Finally, we project the full augmented solution onto the convex,
closed dual-feasible set to obtain $(\hat{\mu},\hat{\eta},\hat{\nu},\hat{g})$.

\subsubsection{Reduced $\ell_1$ Regularized Convex Optimization}
Given the feasible dual estimate \((\hat{\mu},\hat{\eta},\hat{\nu},\hat{g})\) and dual class predictions \((\tilde{\mu},\tilde{g})\), we apply Theorem~\ref{th:screen} to \emph{safely} prune variables and constraints.
Let \(n_\theta \coloneqq \dim(\boldsymbol{\theta}_t)\) and  
\begin{equation} \label{I}
    \mathcal{I}= \{i\in\mathbb{I}_1^{n_\theta} \;|\; \lVert\hat{g}_i\rVert_\infty +\text{gap}(\hat{\mu},\hat{\eta},\hat{\nu},\hat{g})< \lambda \wedge \Tilde{g}_i =0\},
\end{equation}
which is the index set of variables satisfying the screening condition in Theorem~\ref{th:screen}.


Let $r \coloneqq |\mathcal{I}|$. We write $\bar{\boldsymbol{\theta}}_t = S_\mathcal{I}\boldsymbol{\theta}_t$
and $\boldsymbol{\theta}_t = E_\mathcal{I}\bar{\boldsymbol{\theta}}_t$ where
\begin{equation}\label{eq:SI}
(S_\mathcal{I})_{k,i} =
\begin{cases}
1, & i = \mathcal{I}(k),\\
0, & \text{otherwise},
\end{cases}
\qquad
S_\mathcal{I}\in\{0,1\}^{r\times n_\theta},
\end{equation}
and the embedding matrix
\begin{equation}\label{eq:EI}
E_\mathcal{I} \coloneqq S_\mathcal{I}^\top \in \{0,1\}^{n_\theta\times r}.
\end{equation}

The reduced optimization over \(\boldsymbol{\Bar{\theta}_t}\in\mathbb{R}^{r}\) is
{\small{
\begin{subequations}\label{opt:l1_reduced}
\begin{align}
 \min_{\boldsymbol{\Bar{\theta}_t}}&\quad f^0_t(E_\mathcal{I}\Bar{\boldsymbol{\theta}}_t) + \lambda \lVert SE_\mathcal{I}\Bar{\boldsymbol{\theta}}_t\rVert_1
 \label{eq:l1_theta_cost_reduced}\\
 \text{s.t. }&\quad \Bar{f}^i_t(E_\mathcal{I}\Bar{\boldsymbol{\theta}}_t) \leq 0, \quad \forall i\in\mathbb{I}_1^{m-c} \\
 &\quad \Tilde{f}^i_t(E_\mathcal{I}\Bar{\boldsymbol{\theta}}_t) \leq -\zeta, \quad  \forall i\in \mathcal{K} \\
 &\quad h^i_t(E_\mathcal{I}\Bar{\boldsymbol{\theta}}_t) = 0, \quad \forall i\in\mathbb{I}_1^{p},
\end{align}
\end{subequations}}}where the reduced optimization problem includes only the relevant decision variables and constraints denoted by $\mathcal{I}$ and
\begin{equation} \label{K}
    \mathcal{K}=\{i\in\mathbb{I}_1^{c} \;|\; \lVert \hat{\mu}_i \rVert_2 + gap(\hat{\mu},\hat{\eta},\hat{\nu},\hat{g}) \leq \epsilon/\zeta \wedge \Tilde{\mu}_i=0\}.
\end{equation}

\begin{algorithm}[t]
\caption{SHIELD: Safe Hierarchical Inference for Lightweight Duality-screened MPC}
\label{alg:hmpc}
\begin{algorithmic}[1]  
\Require \eqref{opt:l1_epigraph}, \eqref{opt:dual_problem}, and $\pi^{\text{class}}_{\zeta},\epsilon,\zeta$. 
\State $(\tilde{\mu},\tilde{g}) \gets \pi^{\text{class}}_{\zeta}(\mathbf{z}_t)$ [Dual Classification]
\State Construct $S_\mu^{r}$ and $S_g^{r}$ using \eqref{eq:classification} [Selection Matrix]
\State $(\mu^{r},g^{r}) \gets (S_\mu^{r}\mu,\, S_g^{r}g)$ using \eqref{eq:reduced-duals} [Dual Sparsification]
\State Solve the \emph{unconstrained} version of \eqref{opt:dual_problem} in $(\mu^{r},\eta,\nu,g^{r})$ and lift $(\mu^{r},g^{r})$ to full dimension via \eqref{eq:reduced-duals} [Linear Eq. Solve]
\State $(\hat{\mu},\hat{\eta},\hat{\nu},\hat{g}) \gets [(\mu,\eta,\nu,g)]_{\mu\ge 0,\ \eta\ge 0,\ \lVert g\rVert_\infty \le \lambda}$ [Projection]
\State $\mathcal{I} \gets $ \eqref{I}, $\mathcal{K} \gets$ \eqref{K}, $E_{\mathcal I}\gets \eqref{eq:EI}$
\State Formulate \eqref{opt:l1_reduced} using $E_{\mathcal I}$, $\mathcal{K}$, and \eqref{opt:l1_epigraph} [Reduced Problem]
\State \textbf{Output:} $\boldsymbol{\Bar{\theta}^\star_t}\gets$ Solve \eqref{opt:l1_reduced}
\end{algorithmic}
\end{algorithm}

\section{Example: Stochastic MPC with Multi-Modal Predictions} \label{sec:example}
In this section, we implement Algorithm~\ref{alg:hmpc} to accelerate a convex SMPC reformulation with multi-modal, uncertain predictions of surrounding vehicles. As the numbers of modes ($M$) and vehicles ($V$) increase, constraints and decision variables scale combinatorially~\cite{nair_mm_mpc,kim_iv}. We use $N=14$, $M=2$, $V=3$, and $\Delta t=0.1\,\mathrm{s}$, yielding $312$ collision-avoidance constraints and $234$ decision variables, challenging high-frequency computation.
In simulation, we compare two policies with the same parameters:
(i) \emph{Full MPC}—solving the original convex program~\eqref{opt:l1_epigraph}; and
(ii) \emph{Reduced MPC}—solving the screened/reduced program~\eqref{opt:l1_reduced}. We implement and evaluate our policies in the nuPlan simulator~\cite{nuplan}.

\subsection{Multi-modal Trajectory Prediction Model} 
The open-source implementation of Wayformer \cite{wayformer}, an attention-based multi-modal trajectory prediction model proposed by Waymo, was trained using the nuPlan Mini v1.2 dataset \cite{nuplan} using the UniTraj framework \cite{feng2024unitraj}.

\subsection{$\ell_1$ Regularized Convex MPC Formulation}
We add the $\ell_1$ regularization term to the second-order cone program (SOCP) from \cite[Eq. (6)]{kim_iv} to promote decision variable sparsity and enable screening. Given the multi-modal prediction of $V$ vehicles, the resulting SOCP over $\boldsymbol{\theta}_t$ is 

{\small
{\begin{align}\label{opt:Primal_MPC}
\begin{aligned} \min_{\boldsymbol{\theta}_t}~~~&\frac{1}{2}||\mathcal{Q}_t\boldsymbol{\theta}_t||_2^2+\mathcal{C}^\top_t\boldsymbol{\theta}_t +\lambda\sum_{\substack{k\in\mathbb{I}_{t+1}^{t+N},~i\in\mathbb{I}_1^{V}}}||\mathcal{S}^{i}_{k|t}\boldsymbol{\theta}_t||_1\\
\text{s.t}&~\mathcal{A}_t\boldsymbol{\theta}_t+\mathcal{R}_t\in\mathbb{K}:=(\bigotimes_{s=1}^{n_c}\mathbb{K}_s)\times \mathbb{K}_{xu}
\end{aligned}
\end{align}}}where $\mathcal{A}_t,\mathcal{R}_t,\mathcal{Q}_t,\mathcal{C}_t$ denote the prediction models and constraints; $\mathcal{S}_{k|t}^{i}$ is the selection matrix that collects the affine disturbance feedback gain terms in $\boldsymbol{\theta}_t$~\cite{nair_mm_mpc}. $n_c = \mathcal{O}(N\cdot V \cdot M^{V})$, $\mathcal{Q}_t \succ 0$. The $\mathbb{K}_s$ (for all $s\in\mathbb{I}_1^{n_c}$) are cones associated with collision-avoidance (\emph{screenable}) constraints, while $\mathbb{K}_{xu}$ is the cone for all state–input (\emph{immutable}) constraints. See~\cite{kim_iv} for further details.
The dual SOCP of \eqref{opt:Primal_MPC} is

{\small{
\begin{equation} \label{opt:dual_cost}
    \tfrac{1}{2}\Big\|
\mathcal{Q}_t^{-1}\!\Big(
\mathcal{A}_t^\top
\begin{bmatrix}\boldsymbol{\mu}_t\\ \boldsymbol{\eta}_t\end{bmatrix}
-\mathcal{C}_t
-\lambda \!\sum_{k=t+1}^{t+N}\sum_{i=1}^{V}\!
\mathcal{S}^i_{k|t}\, g^i_{k|t}
\Big)\Big\|_2^2 + \mathcal{R}_t^\top\begin{bmatrix}\boldsymbol{\mu}_t\\ \boldsymbol{\eta}_t\end{bmatrix} 
\end{equation}}}
{\small{
\begin{equation}
\label{opt:Dl1_MPC}
\begin{aligned}
&\min_{\boldsymbol{\mu}_t,\boldsymbol{\eta}_t,\{g^i_{k|t}\}}~~
\eqref{opt:dual_cost}\\
\text{s.t.}~~
& \boldsymbol{\mu}_t \in \bigotimes_{s=1}^{n_c}\mathbb{K}_s^*,\quad
\boldsymbol{\eta}_t \in \mathbb{K}_{xu}^*,\\
&\|g^i_{k|t}\|_\infty \le \lambda,\ \forall\, k\in\mathbb{I}_{t+1}^{t+N}, i \in\mathbb{I}_{1}^{V}.
\end{aligned}
\end{equation}}}

\subsection{Dual Classifier: Training Data Collection}
We use nuPlan \cite{nuplan} to simulate traffic scenarios and collect training data as follows: First, given the scene encoding of road agent history, traffic lights, and a vectorized map, Wayformer \cite{wayformer} forecasts the multi-modal trajectory of surrounding agents. Then, we formulate \eqref{opt:Primal_MPC} in CasADi~\cite{andersson2019casadi}, replacing $\mathbb{K}_s$ with $\zeta$-tightened cones ($\zeta = 0.5$) and solve using Gurobi~\cite{gurobi}. The feature is $\small{\mathbf{z}_t=\big\{\mathbf{o}_{t}^{\,i}-\mathbf{x}_{t}^{EV}\big\}_{i=1}^{V},}$where $\mathbf{o}_{t}^{\,i}$ denotes the $i$-th vehicle’s multi-modal predictions (top $M$ modes over horizon $N$) and $\mathbf{x}_{t}^{EV}$ is the ego planned trajectory.
At each time \(t\), we log features \(\mathbf{z}_t\) together with the optimal dual variables, and convert them into class labels \((\tilde{\mu}_t^\star,\tilde{g}_t^\star)\).
We collect over \(33{,}000\) (features, labels) pairs across \(300\) scenarios from the nuPlan Mini v1.2 dataset~\cite{nuplan}, and split the data \(85\) / \(15\%\) for training and evaluation, respectively.

\subsection{Dual Classifier: Architecture Design}
The classifier uses a Transformer-style~\cite{attention}, decoder-only architecture \emph{without positional encodings} to preserve permutation equivariance to vehicle ordering. Input features $\mathbf{z}_t$ are linearly embedded and processed by multi-head attention blocks with layer normalization, residual connections, and feedforward layers capture inter-vehicle interactions. Decoder is unrolled for $N$ steps to produce class-1 probabilities (as in RAID-Net~\cite{kim_iv}), modeling temporal dependencies via recurrence. Training uses weighted binary cross-entropy with AdamW~\cite{adamw} optimizer and cosine-annealing schedule~\cite{kim_iv}.

\subsection{Numerical Results} \label{sec:results}
\subsubsection{Classifier Evaluation} 

We evaluate on the \emph{evaluation split}, comparing normalized BCE of $\pi^{\mathrm{class}}_\zeta$ (1.1M) against an MLP $\pi^{\mathrm{MLP}}_\zeta$ (1.4M) and RAID-Net $\pi^{\mathrm{RAIDN}}_\zeta$ (4.6M) in Fig.~\ref{fig:joint_metrics}. $\pi^{\mathrm{class}}_\zeta$ attains lowest normalized BCE. Its confusion matrix (Fig.~\ref{fig:joint_metrics}) shows recall $0.99$, precision $0.97$, and accuracy $0.96$, indicating few false negatives and improved interaction modeling over baselines under multi-modal predictions.
\begin{figure} 
    \centering
    \includegraphics[width=\linewidth]{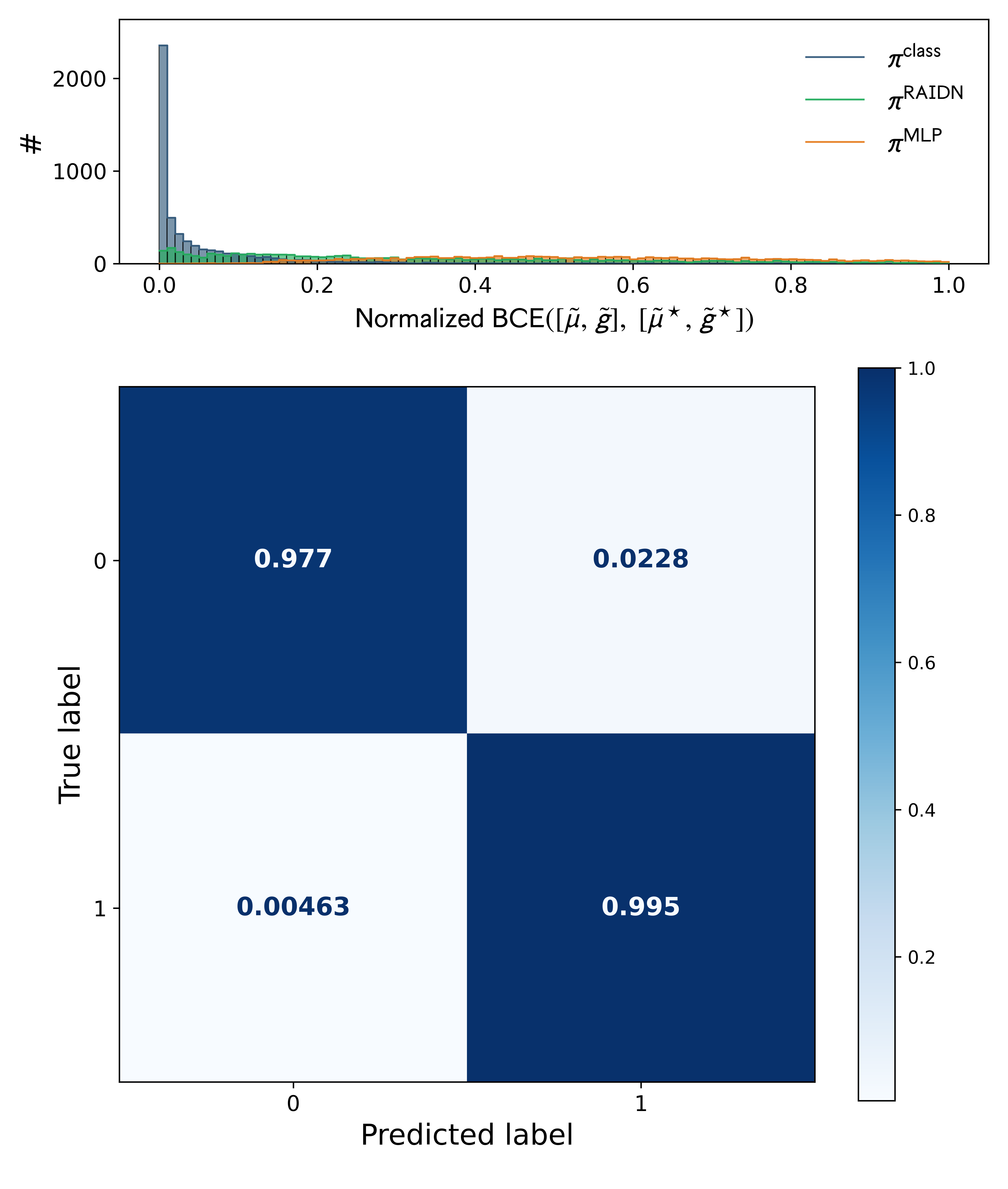} 
    \caption{Per-sample normalized binary cross-entropy (top) and normalized confusion matrix (bottom) for the $(\mu,g)$-dual classification, evaluated on the evaluation data split. We compare our classifier $\pi^{\mathrm{class}}_\zeta$ with an MLP $\pi^{\mathrm{MLP}}_\zeta$ and RAID-Net $\pi^{\mathrm{RAIDN}}_\zeta$ \cite{kim_iv}.}
    \label{fig:joint_metrics}
\end{figure}

\subsubsection{Closed-loop Simulation} \label{sec:cl_sim}
In nuPlan~\cite{nuplan}, we evaluate \emph{Full MPC}~\eqref{opt:Primal_MPC} (not $\zeta$ tightened) and \emph{Reduced MPC} (Alg.~\ref{alg:hmpc}) with $\epsilon=0.01$, a conservative underestimate of $\epsilon_{\text{crit}}$ in Proposition~\ref{prop1}. We run 25{,}000 solves across $150+$ traffic scenarios (e.g., lane merges, unprotected left turns). Table~\ref{tab:comparison_tab} reports feasibility, collision rate, average fraction of constraints enforced (constraint keep-rate in \emph{Reduced MPC}), and average fraction of ADF gains kept. We also report solve times and auxiliary runtimes: classifier inference, CasADi parametric calls (dual construction/solve), optimal-dual estimation, and duality-gap computation in SHIELD. These auxiliaries are not yet optimized and can be reduced further. All MPC formulations evaluated in this work, both the full and reduced problems, are warm-started using the shifted solution from the previous time step. The computational gains reported, therefore, reflect the benefit of constraint and variable screening beyond what warm-starting alone provides: the full MPC baseline already benefits from warm-starting, yet its solve times remain substantially higher than the screened formulation, confirming that the bottleneck is problem size rather than initialization quality.

\begin{figure}
    \centering
    \includegraphics[width=\linewidth]{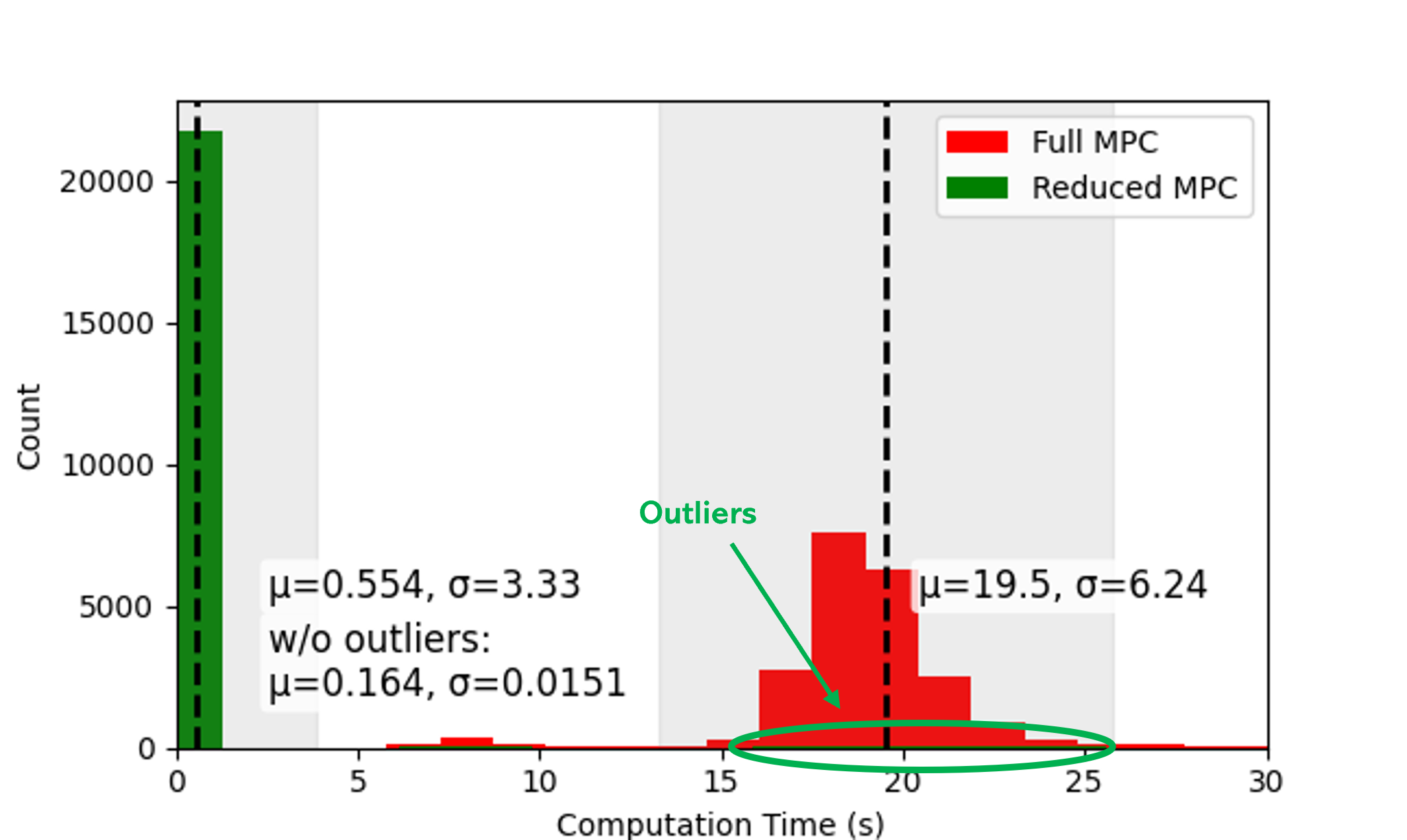}
    \caption{Distributions of total computation times for \emph{Full MPC} \eqref{opt:Primal_MPC} and \emph{Reduced MPC} (Alg.~\ref{alg:hmpc}) over 25{,}000 nuPlan steps (SMPC with multi-modal prediction). Shaded bands show $\pm 1\sigma$. The circled region contains outliers for \emph{Reduced MPC} computation time ($\sim 6\%$ of total count).}
    \label{fig:solve_time}
\end{figure}
As summarized in Table~\ref{tab:comparison_tab}, both policies remained feasible and incurred zero collisions. Under the proposed method, the \emph{Reduced MPC} enforced $6.80\%$ of collision-avoidance constraints on average while retaining $3.19\%$ of the ADF gains. Average solve time decreased to $0.498 \pm 3.33~\mathrm{s}$ for the \emph{Reduced MPC}, compared to $19.53 \pm 6.24~\mathrm{s}$ for the \emph{Full MPC}. The large variance is driven by rare outliers ($\sim6\%$) where the dual estimate prevents screening and the solve reverts to near-full size, which may happen when the $\pi^{\text{class}}_{\zeta}$ predictions are inaccurate. In a small fraction of time steps, the screening procedure removes few or no constraints, resulting in solve times comparable to the full MPC. These cases arise primarily when the ego vehicle is in close proximity to multiple interacting agents simultaneously, such that many scenario-indexed collision avoidance constraints are simultaneously active or near-active. In such configurations, the dual gap bound is insufficiently tight to certify inactivity, because the true optimal dual variables associated with these constraints are large. Geometrically, these correspond to states where the ego vehicle's feasible set is tightly constrained from multiple directions — precisely the situations where constraint reduction is least available regardless of method. Notably, these outlier steps do not compromise safety: when screening is inconclusive, the method defaults to solving the full problem, so the worst-case behavior is equivalent to standard MPC. The computational overhead of the failed screening attempt (dual estimation and screening test) is small relative to the QP solve, so the cost of conservatism in these cases is marginal.

Excluding outliers, the \emph{Reduced MPC} achieves $0.109 \pm 0.01~\mathrm{s}$; including auxiliary overhead (Fig.~\ref{fig:solve_time}), it is $35\times$ faster overall and $119\times$ faster without outliers.

\begin{table}[!t]
  \centering
  \caption{Closed-loop Simulation Results}
  \label{tab:comparison_tab}
  \begin{threeparttable}
  \setlength{\tabcolsep}{4pt}
    \begin{tabular}{@{}lcc@{}}
    \toprule
    \textbf{Performance Metric} &
    \makecell{\textbf{Full MPC}\\\textbf{\eqref{opt:Primal_MPC}}} &
    \makecell{\textbf{Reduced MPC}\\\textbf{Alg.~\ref{alg:hmpc}}} \\
    \midrule
    Feasibility (\%)                 & 100                         & 100                                   \\
    Collision (\%)                   & 0                           & 0                                     \\
    Avg.\ Constraints Enforced (\%)  & 100                         & $\mathbf{6.80\pm16.44}$                         \\
    Avg.\ ADF Gain Kept (\%)  & 100                         & $\mathbf{3.19 \pm 13.81}$                         \\
    Avg.\ Solve Time (Gurobi) [s]    & $19.53 \,\pm\, 6.24$        & $ ^\dag\mathbf{0.498 \,\pm\, 3.33}$ \\
    Avg.\ CasADi Func. Calls [s]  & \textemdash                 & $0.0345 \,\pm\, 0.0159$               \\    
    Avg.\ Classifier Query Time [s]  & \textemdash                 & $0.0049 \,\pm\, 0.0031$               \\
    Avg.\ Dual Approx.\ Time [s]     & \textemdash                 & $0.0083 \,\pm\, 0.0041$               \\
    Avg.\ Total Computation Time [s] & $19.53 \,\pm\, 6.24$        & $^*\mathbf{0.554 \,\pm\, 3.33}$         \\
    \bottomrule
  \end{tabular}
  \begin{tablenotes}\footnotesize
      \item[\dag]Total solve time excluding outliers: $0.109 \,\pm\, 0.01$ s.
    \item[*]Total computation time excluding outliers: $0.164 \,\pm\, 0.0151$ s.
  \end{tablenotes}
  \end{threeparttable}

\end{table}

\begin{figure*}[!t]
  \includegraphics[width=\textwidth]{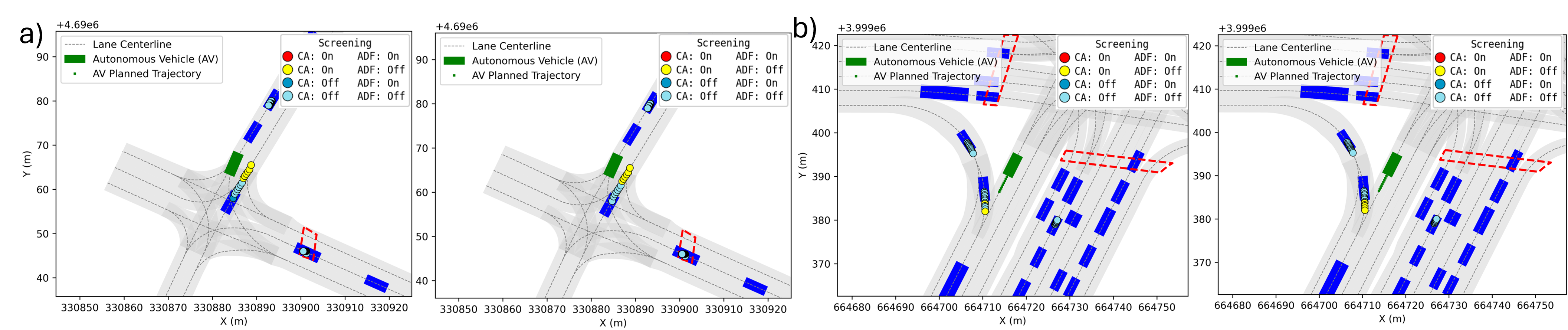}
  \caption{Two sampled real-world scenarios from nuPlan controlled by the \emph{Reduced MPC} (left) and the \emph{Full MPC} (right): (a) unprotected left turn, and (b) lane merge. Colored circles denote multi-modal predictions from a Wayformer model~\cite{wayformer}, with screening status indicated by color: \emph{red} (CA: On, ADF: On), \emph{yellow} (CA: On, ADF: Off), \emph{teal} (CA: Off, ADF: On), and \emph{light blue} (CA: Off, ADF: Off). Red dashed outlines mark the target vehicles for which collision avoidance constraints are retained in the \emph{Reduced MPC}
}

  \label{fig:closedloop}
\end{figure*}

Furthermore, the dual-approximation time (i.e., solving the unconstrained dual problem) using~\eqref{eq:reduced-duals} with classifier-selected variables is $0.0083 \pm 0.0041\,\mathrm{s}$, compared to $0.0182 \pm 0.0087\,\mathrm{s}$ for the full-dimensional unconstrained dual solve without the classifier. This speedup is critical for real-time control and highlights the benefit of data-driven dimension reduction for high-dimensional dual approximation.

Fig.~\ref{fig:closedloop} illustrates the constraint screening and ADF gain reduction in two real-world nuPlan scenarios: (a) an unprotected left turn and (b) a lane merge. In both cases, the left column shows the \emph{Reduced MPC}  solution and the right column shows the \emph{Full MPC} solution, with colored  circles indicating the screening status of each predicted mode for each target  vehicle. In scenario (a), the AV must navigate an unprotected left turn in the presence of oncoming traffic. The \emph{Reduced MPC} retains only the collision avoidance constraints corresponding to the modes that are active in the \emph{Full MPC} solution, indicated by the red and yellow circles. The light blue circles denote screened-out modes whose constraints and ADF gains are deemed inactive, confirming that the screening criterion correctly identifies the safety-relevant subset of the multi-modal predictions without discarding any constraints that would affect the AV's planned trajectory.

In scenario (b), the AV must merge into an adjacent lane occupied by multiple target vehicles traveling at highway speeds. The conflict zone spans a larger spatial region, and multiple target vehicles contribute active constraints. The \emph{Reduced MPC} retains a subset of the \emph{Full MPC} active constraint set, including the entry and exit timing constraints that govern when the AV commits to and completes the merge. These boundary constraints largely determine the yield behavior and gap acceptance decision, such that the \emph{Reduced MPC} closely reproduces the \emph{Full MPC} closed-loop trajectory while operating on a significantly smaller constraint set. In both scenarios, the screening procedure preserves the safety-critical constraints and ADF gains while eliminating redundant decision variables, reducing computational load without degrading closed-loop performance. To further quantify the closed-loop similarity between the \emph{Reduced MPC} and \emph{Full MPC}, we compute the Average Displacement Error (ADE) and Final Displacement Error (FDE) between their respective planned trajectories. In scenario (a), the \emph{Reduced MPC} achieves an ADE of $0.15\,\mathrm{m}$ and FDE of $0.88\,\mathrm{m}$, while in scenario (b) the trajectories are nearly identical, with an ADE of $0.006\,\mathrm{m}$ and FDE of $0.018\,\mathrm{m}$. These results confirm that despite retaining only $6.80\%$ of the collision avoidance constraints and $3.19\%$ of the ADF gains on average, the \emph{Reduced MPC} closely replicates the \emph{Full MPC} closed-loop behavior at $35\times$ lower computation time. Code and videos are available at: \url{https://github.com/MPC-Berkeley/SHIELD_mpc}.

The screening performance depends on three parameters: the safety margin $\zeta$, the cost tolerance $\epsilon$, and the sparsity weight $\lambda$. Increasing $\zeta$ enlarges the feasibility buffer in Proposition~1, making screening easier to certify, but tightens the constraints seen by the solver, which may increase conservatism or cause infeasibility when the original problem is already tight. Decreasing $\epsilon$ strengthens the certification requirement in Theorem~\ref{th:screen}, reducing the set of screenable constraints but providing a harder safety guarantee. The sparsity weight $\lambda$ controls the variable screening threshold: larger $\lambda$ drives more policy parameters to zero, increasing the number of removable variables but potentially degrading control performance. Table~\ref{tab:sensitivity} reports the sensitivity of SHIELD to $\epsilon$ and $\lambda$ over 20 scenarios; $\zeta$ is excluded from the sweep because it is a problem-level parameter encoding the required physical safety margin for the application, set from domain knowledge rather than algorithmic tuning. Varying $\zeta$ would change the safety specification itself rather than probe the algorithm's internal sensitivity, and since $\zeta$ directly enters the feasibility buffer in Proposition~1, increasing it may render some scenarios infeasible not due to algorithm failure but due to the tightened constraint set, conflating algorithm performance with problem feasibility. For $\epsilon$, tightening from $0.1$ to $0.001$ increases the constraint keep rate from $5.9\%$ to $52.4\%$, confirming that stricter certification retains more constraints, but feasibility and safety remain intact across all values. For $\lambda$, reducing from the nominal $100$ to $10$ raises the constraint keep rate from $5.6\%$ to $35.9\%$, as fewer variables are driven to zero, leaving less opportunity for screening. Across all tested values, feasibility remains at $100\%$ and collision rate at $0\%$, indicating that the safety guarantees are robust to parameter variation. In our experiments, $\zeta = 0.5$ and $\epsilon = 0.01$ were chosen conservatively and held fixed across all $150+$ scenarios without tuning. The primary failure mode occurs when the dual-gap bound is too loose to certify screening, typically in dense multi-agent configurations where many constraints are simultaneously near-active. In these cases, the algorithm defaults to the full problem, so safety is preserved, but no computational benefit is realized.

\begin{table}[h]
\centering
\caption{Parameter sensitivity of SHIELD averaged over 20 scenarios.}
\label{tab:sensitivity}
\resizebox{\columnwidth}{!}{%
\begin{tabular}{lrrrrr}
\toprule
Param. & Constr. (\%) & ADF (\%) & Time (s) & Feasible (\%) & Coll. (\%) \\
\midrule
\multicolumn{6}{l}{\textit{Sweep $\varepsilon$}}\\ 
0.001 & $52.4 \pm 10.9$ & $1.2 \pm 0.6$ & $0.160 \pm 0.052$ & 100.0 & 0.0 \\
0.005 & $12.1 \pm 9.1$  & $1.2 \pm 0.6$ & $0.115 \pm 0.007$ & 100.0 & 0.0 \\
\textbf{0.010} & $12.2 \pm 7.7$ & $1.1 \pm 0.6$ & $0.113 \pm 0.004$ & 100.0 & 0.0 \\
0.050 & $8.1 \pm 5.7$   & $1.0 \pm 0.6$ & $0.112 \pm 0.003$ & 100.0 & 0.0 \\
0.100 & $5.9 \pm 3.6$   & $1.0 \pm 0.6$ & $0.111 \pm 0.003$ & 100.0 & 0.0 \\
\midrule
\multicolumn{6}{l}{\textit{Sweep $\lambda$}}\\ 
10  & $35.9 \pm 4.4$ & $1.4 \pm 0.9$ & $0.129 \pm 0.017$ & 100.0 & 0.0 \\
50  & $9.6 \pm 7.5$  & $1.2 \pm 0.7$ & $0.121 \pm 0.019$ & 100.0 & 0.0 \\
\textbf{100} & $5.6 \pm 1.9$ & $1.4 \pm 0.8$ & $0.125 \pm 0.021$ & 100.0 & 0.0 \\
150 & $4.9 \pm 1.0$  & $1.3 \pm 0.7$ & $0.112 \pm 0.005$ & 100.0 & 0.0 \\
200 & $8.3 \pm 0.7$  & $1.9 \pm 1.4$ & $0.173 \pm 0.067$ & 100.0 & 0.0 \\
\midrule
\bottomrule
\end{tabular}}
\vspace{4mm}
{\footnotesize $^*$\textbf{Bold} values indicate nominal parameters used in the main experiments.}
\end{table}

\section{Conclusion}
We proposed a safe hierarchical inference framework for lightweight duality-based screening that provably reduces the decision dimension and constraint set of an $\ell_1$-regularized convex program, enabling high-rate computation with $\zeta$-tightening. Leveraging strong convexity and Lagrangian duality, we derived certification conditions to satisfy the \emph{original} constraints while removing irrelevant constraints and variables. We validated the approach on stochastic MPC with multi-modal predictions in complex, interactive traffic scenarios. The safe screening certificates derived in this work rely on strong convexity of the objective, Lipschitz continuity of the gradient, and Slater's constraint qualification. These assumptions hold for the quadratic-cost, linearly-constrained MPC formulations considered in our experiments and are satisfied by a broad class of practical MPC problems that use linearized dynamics and convex constraint sets. However, they exclude nonconvex formulations arising from nonlinear vehicle dynamics, non-convex collision geometry, or mixed-integer decision variables such as pass/stop maneuver selection. Extending the certification framework to these  settings is nontrivial: without strong duality, the gap bound in 
Theorem~\ref{th:screen} does not hold, and the screening test may produce false negatives (failing to screen) or, more critically, could not be applied at all without alternative bounding mechanisms. Possible avenues include applying screening after convexification within a sequential convex programming loop, or deriving relaxation-based dual bounds for mixed-integer programs, though both directions require substantial theoretical development and are left to future work. Beyond extending the theoretical framework, two practical directions remain. First, validating SHIELD on diverse large-scale convex programs outside MPC—such as network flow, sparse portfolio optimization, and sensor selection—would demonstrate the generality claimed by the problem formulation in Sec.~III. Second, a systematic sensitivity analysis of $(\zeta, \epsilon, \lambda)$ across problem classes and conditioning regimes is needed to characterize how screening rates, conservatism, and failure frequency depend on these parameters, and to understand the robustness and generalizability of the algorithm across different problem settings.

\bibliographystyle{IEEEtran}
\bibliography{references.bib}

@article{andersson2019casadi,
  title={CasADi: a software framework for nonlinear optimization and optimal control},
  author={Andersson, Joel AE and Gillis, Joris and Horn, Greg and Rawlings, James B and Diehl, Moritz},
  journal={Mathematical Programming Computation},
  year={2019},
  publisher={Springer}}

@article{wang2014iteration,
  author  = {Po-Wei Wang and Chih-Jen Lin},
  title   = {Iteration Complexity of Feasible Descent Methods for Convex Optimization},
  journal = {Journal of Machine Learning Research},
  year    = {2014},
  volume  = {15},
  number  = {45},
  pages   = {1523--1548}, 
}

@book{boyd2004convex,
  title={Convex optimization},
  author={Boyd, Stephen P and Vandenberghe, Lieven},
  year={2004},
  publisher={Cambridge University Press}
}

@INPROCEEDINGS{kim_iv,
  author={Kim, Hansung and Nair, Siddharth H. and Borrelli, Francesco},
  booktitle={2024 IEEE Intelligent Vehicles Symposium (IV)}, 
  title={Scalable Multi-modal Model Predictive Control via Duality-based Interaction Predictions}, 
  year={2024},
  volume={},
  number={},
  pages={1499-1504}}

@INPROCEEDINGS{drobustmpc,
  author={Hernandez, Bernardo and Trodden, Paul},
  booktitle={2016 UKACC 11th International Conference on Control (CONTROL)}, 
  title={Distributed model predictive control using a chain of tubes}, 
  year={2016},
  volume={},
  number={},
  pages={1-6}}

@INPROCEEDINGS{conflict_mpc,
  author={Kottinger, Justin and Almagor, Shaull and Lahijanian, Morteza},
  booktitle={2022 IEEE/RSJ International Conference on Intelligent Robots and Systems (IROS)}, 
  title={Conflict-Based Search for Multi-Robot Motion Planning with Kinodynamic Constraints}, 
  year={2022},
  volume={},
  number={},
  pages={13494-13499}}

@ARTICLE{centralized_multi,
  author={Li, Bai and Ouyang, Yakun and Zhang, Youmin and Acarman, Tankut and Kong, Qi and Shao, Zhijiang},
  journal={IEEE Robotics and Automation Letters}, 
  title={Optimal Cooperative Maneuver Planning for Multiple Nonholonomic Robots in a Tiny Environment via Adaptive-Scaling Constrained Optimization}, 
  year={2021},
  volume={6},
  number={2},
  pages={1511-1518}}

@ARTICLE{nair_mm_mpc,
  author={Nair, Siddharth H. and Lee, Hotae and Joa, Eunhyek and Wang, Yan and Tseng, H. Eric and Borrelli, Francesco},
  journal={IEEE Transactions on Control Systems Technology}, 
  title={Predictive Control for Autonomous Driving With Uncertain, Multimodal Predictions}, 
  year={2025},
  volume={33},
  number={4},
  pages={1178-1192}}

@article{candes_enhancing_2008,
	title = {Enhancing {Sparsity} by {Reweighted} $\ell_1$ {Minimization}},
	volume = {14},
	issn = {1531-5851},
	number = {5},
	journal = {Journal of Fourier Analysis and Applications},
	author = {Candès, Emmanuel J. and Wakin, Michael B. and Boyd, Stephen P.},
	month = dec,
	year = {2008},
	pages = {877--905},
}

@misc{gurobi,
  author = {{Gurobi Optimization, LLC}},
  title = {{Gurobi Optimizer Reference Manual}},
  year = 2024,
  url = "https://www.gurobi.com"
}

@INPROCEEDINGS{nuplan, 
  title={NuPlan: A closed-loop ML-based planning benchmark for autonomous vehicles},
  author={H. Caesar and J. Kabzan},
  booktitle={CVPR ADP3 workshop},
  year=2021
}

@article{feng2024unitraj,
  title={UniTraj: A Unified Framework for Scalable Vehicle Trajectory Prediction},
  author={Feng, Lan and Bahari, Mohammadhossein and Amor, Kaouther Messaoud Ben and Zablocki, {\'E}loi and Cord, Matthieu and Alahi, Alexandre},
  journal={arXiv preprint arXiv:2403.15098},
  year={2024}
}

@misc{wayformer,
  

  url = {https://arxiv.org/abs/2207.05844},

  author = {Nayakanti, Nigamaa and Al-Rfou, Rami and Zhou, Aurick and Goel, Kratarth and Refaat, Khaled S. and Sapp, Benjamin},

  title = {Wayformer: Motion Forecasting via Simple \& Efficient Attention Networks},

  publisher = {arXiv},

  year = {2022},

  copyright = {arXiv.org perpetual, non-exclusive license}
}

@misc{adamw,
      title={Decoupled Weight Decay Regularization}, 
      author={Ilya Loshchilov and Frank Hutter},
      year={2019},
      eprint={1711.05101},
      archivePrefix={arXiv},
      primaryClass={cs.LG},
      url={https://arxiv.org/abs/1711.05101}, 
}

@inproceedings{attention,
 author = {Vaswani, Ashish and Shazeer, Noam and Parmar, Niki and Uszkoreit, Jakob and Jones, Llion and Gomez, Aidan N and Kaiser, \L ukasz and Polosukhin, Illia},
 booktitle = {Advances in Neural Information Processing Systems},
 pages = {},
 publisher = {Curran Associates, Inc.},
 title = {Attention is All you Need},
 volume = {30},
 year = {2017}
}

@misc{arango,
      title={Learning-Based Approximate Nonlinear Model Predictive Control Motion Cueing}, 
      author={Camilo Gonzalez Arango and Houshyar Asadi and Mohammad Reza Chalak Qazani and Chee Peng Lim},
      year={2025},
      eprint={2504.00469},
      archivePrefix={arXiv},
      primaryClass={cs.RO},
      url={https://arxiv.org/abs/2504.00469}, 
}

@inproceedings{Sacks_2022,
   title={Learning to Optimize in Model Predictive Control},
  
   booktitle={2022 International Conference on Robotics and Automation (ICRA)},
   publisher={IEEE},
   author={Sacks, Jacob and Boots, Byron},
   year={2022},
   month=may, pages={10549–10556} }

@article{Stellato_2017,
   title={High-Speed Finite Control Set Model Predictive Control for Power Electronics},
   volume={32},
   ISSN={1941-0107},
   
   number={5},
   journal={IEEE Transactions on Power Electronics},
   publisher={Institute of Electrical and Electronics Engineers (IEEE)},
   author={Stellato, Bartolomeo and Geyer, Tobias and Goulart, Paul J.},
   year={2017},
   month=may, pages={4007–4020} }

@INPROCEEDINGS{quirynen,
  author={Chakrabarty, Ankush and Quirynen, Rien and Romeres, Diego and Di Cairano, Stefano},
  booktitle={2021 IEEE International Conference on Systems, Man, and Cybernetics (SMC)}, 
  title={Learning Disagreement Regions with Deep Neural Networks to Reduce Practical Complexity of Mixed-Integer MPC}, 
  year={2021},
  volume={},
  number={},
  pages={3238-3244}}

@ARTICLE{Marcucci,
  author={Marcucci, Tobia and Tedrake, Russ},
  journal={IEEE Transactions on Automatic Control}, 
  title={Warm Start of Mixed-Integer Programs for Model Predictive Control of Hybrid Systems}, 
  year={2021},
  volume={66},
  number={6},
  pages={2433-2448}}

@ARTICLE{lampos,
  author={Russo, Luigi and Nair, Siddharth H. and Glielmo, Luigi and Borrelli, Francesco},
  journal={IEEE Control Systems Letters}, 
  title={Learning for Online Mixed-Integer Model Predictive Control With Parametric Optimality Certificates}, 
  year={2023},
  volume={7},
  number={},
  pages={2215-2220}}

@misc{cauligi2022learningmixedintegerconvexoptimization,
      title={Learning Mixed-Integer Convex Optimization Strategies for Robot Planning and Control}, 
      author={A. Cauligi and P. Culbertson and B. Stellato and D. Bertsimas and M. Schwager and M. Pavone},
      year={2022},
      eprint={2004.03736},
      archivePrefix={arXiv},
      primaryClass={cs.RO},
      url={https://arxiv.org/abs/2004.03736}, 
}

@misc{coco,
      title={CoCo: Online Mixed-Integer Control via Supervised Learning}, 
      author={A. Cauligi and P. Culbertson and E. Schmerling and M. Schwager and B. Stellato and M. Pavone},
      year={2021},
      eprint={2107.08143},
      archivePrefix={arXiv},
      primaryClass={cs.RO},
      url={https://arxiv.org/abs/2107.08143}, 
}

@misc{primal-dual,
      title={Near-Optimal Rapid MPC using Neural Networks: A Primal-Dual Policy Learning Framework}, 
      author={Xiaojing Zhang and Monimoy Bujarbaruah and Francesco Borrelli},
      year={2019},
      eprint={1912.04744},
      archivePrefix={arXiv},
      primaryClass={eess.SY},
      url={https://arxiv.org/abs/1912.04744}, 
}

@techreport{elghaoui2010,
	Author = {Laurent {El Ghaoui} and Vivian Viallon and Tarek Rabbani},
	Institution = {EECS Dept., University of California at Berkeley},
	Month = {September},
	Number = {UC/EECS-2010-126},
	Title = {Safe Feature Elimination in Sparse Supervised Learning},
	Year = {2010}}

@article{bonnefoy2015,
   title={Dynamic Screening: Accelerating First-Order Algorithms for the Lasso and Group-Lasso},
   volume={63},
   ISSN={1941-0476},
   url={http://dx.doi.org/10.1109/TSP.2015.2447503},
   DOI={10.1109/tsp.2015.2447503},
   number={19},
   journal={IEEE Transactions on Signal Processing},
   publisher={Institute of Electrical and Electronics Engineers (IEEE)},
   author={Bonnefoy, Antoine and Emiya, Valentin and Ralaivola, Liva and Gribonval, Remi},
   year={2015},
}

@misc{fercoq2015,
      title={Mind the duality gap: safer rules for the Lasso}, 
      author={Olivier Fercoq and Alexandre Gramfort and Joseph Salmon},
      year={2015},
      eprint={1505.03410},
      archivePrefix={arXiv},
      primaryClass={stat.ML},
      url={https://arxiv.org/abs/1505.03410}, 
}

\appendix
\subsection{KKT Conditions} \label{appendix:kkt}
At a primal–dual optimal point $(\theta^\star_t,\mu^\star,\eta^\star,\nu^\star,g_1^\star,g_2^\star,\gamma^\star)$ of the epigraphic form \eqref{opt:l1_epigraph}:
The primal feasibility requires
{\small{\begin{align*}
&\tilde f^i_t(\boldsymbol{\theta}_t^\star)\le -\zeta\; \forall i\in\mathbb{I}_1^c\;
\bar f^i_t(\boldsymbol{\theta}^\star_t)\le 0,\; \forall i\in\mathbb{I}_1^{m-c} \\
&h^i_t(\boldsymbol{\theta}_t^\star)=0, \; \forall i\in\mathbb{I}_1^p\;
S\boldsymbol{\theta}_t^\star - s^\star \le 0,\;
-\,S\boldsymbol{\theta}_t^\star - s^\star \le 0,\;
s^\star \ge 0.
\end{align*}}}
The dual feasibility requires
{\small{\begin{align*}
&\mu_i^\star \ge 0,\; \forall i\in\mathbb{I}_1^c\;
\eta_i^\star \ge 0,\; \forall i\in\mathbb{I}_1^{m-c}\;
g_1^\star \ge 0,\;
g_2^\star \ge 0,\;
\gamma^\star \ge 0,\; \\
&\lambda\mathbf{1} - g_1^\star - g_2^\star - \gamma^\star = 0.
\end{align*}}}
Complementary slackness conditions are {\small{
\begin{align*}
&\mu_i^\star\,(\tilde f^i_t(\boldsymbol{\theta}_t^\star)+\zeta)=0,\; \forall i\in\mathbb{I}_1^c,\;
\eta_i^\star\,\bar f^i(\boldsymbol{\theta}_t^\star)=0,\; \forall i\in\mathbb{I}_1^{m-c},\; \\
&(g_1^\star)^\top(S\boldsymbol{\theta}_t^\star - s^\star)=0,\;
(g_2^\star)^\top(-S\boldsymbol{\theta}_t^\star - s^\star)=0,\;
(\gamma^\star)^\top s^\star=0.
\end{align*}}}
Lagrangian stationarity condition is 
{\small{
\begin{align*}
&0 \in \partial f^0_t(\boldsymbol{\theta}_t^\star)
+ \sum_i \mu_i^\star\,\partial \tilde f^i_t(\boldsymbol{\theta}_t^\star)
+ \sum_i \eta_i^\star\,\partial \bar f^i_t(\boldsymbol{\theta}_t^\star) \\
&+ \sum_i \nu_i^\star\,\partial h^i_t(\boldsymbol{\theta}_t^\star) + S^\top(g_1^\star - g_2^\star).
\end{align*}}}
Let $g^\star := g_1^\star - g_2^\star$, and the $\ell_1$ subgradient coupling is
{\small{\[
g^\star \in \lambda\,\partial\|S\boldsymbol{\theta}_t^\star\|_1\quad \text{or }(g^\star)_j=
\begin{cases}
\lambda\,\mathrm{sign}([S\boldsymbol{\theta}_t^\star]_j), & [S\boldsymbol{\theta}_t^\star]_j\neq 0,\\
\in[-\lambda,\lambda], & [S\boldsymbol{\theta}_t^\star]_j=0.
\end{cases}
\]}}
\noindent\textit{Moreover,} complementary slackness on the epigraph constraints gives
{\small{$s^\star_j = |[S\boldsymbol{\theta}_t^\star]_j|$ for all $j$. Using $g^\star \in \lambda\,\partial\|S\boldsymbol{\theta}_t^\star\|_1$,
\[
(g^\star)_j [S\boldsymbol{\theta}_t^\star]_j
=
\begin{cases}
\lambda\,|[S\boldsymbol{\theta}_t^\star]_j|, & [S\boldsymbol{\theta}_t^\star]_j\neq 0,\\
0, & [S\boldsymbol{\theta}_t^\star]_j=0\ .
\end{cases}
\]}}
Summing over $j$ yields the tight support-function identity
{\small{\[
g^{\star\top} S\boldsymbol{\theta}_t^\star \;=\; \lambda \sum_j |[S\boldsymbol{\theta}_t^\star]_j|
\;=\; \lambda\,\|S\boldsymbol{\theta}_t^\star\|_1.\quad\quad\quad\quad\square
\]}}

\end{document}